\documentclass[final]{cvpr}

\usepackage{times}
\usepackage{epsfig}
\usepackage{graphicx}
\usepackage{amsmath}
\usepackage{amssymb}

\usepackage{graphicx}
\usepackage{amsthm}
\usepackage{color}
\usepackage{xcolor}
\usepackage{multirow}
\usepackage[ruled, vlined, linesnumbered]{algorithm2e}
\usepackage{tabularx}
\usepackage[font=small]{caption}
\usepackage{bm}
\usepackage{comment}
\usepackage{enumitem}
\usepackage{titling}

\definecolor{airforceblue}{rgb}{0.36, 0.54, 0.66}

\newcommand{\hlt}[1]{{\color{airforceblue} #1}}
\newcommand{\commentT}[1]{{}}
\newcommand{\myparagraph}[1]{\medskip\noindent\textbf{#1}}

\newcommand{\X}{\mathbb{X}}
\newcommand{\R}{\mathbb{R}}
\newcommand{\A}{\mathbb{A}}
\newcommand{\B}{\mathbb{B}}
\newcommand{\E}{\mathbb{E}}
\newcommand{\T}{\mathbb{T}}
\newcommand{\cH}{\mathcal{H}}
\newcommand{\cD}{\mathcal{D}}
\newcommand{\cL}{\mathcal{L}}
\newcommand{\gt}{\theta}

\newcommand{\bgt}{{\bm{\gt}}}
\newcommand{\bu}{\bm{u}}

\newcommand{\bw}{\bm{w}}
\newcommand{\bx}{\bm{x}}
\newcommand{\norm}[1]{\left\lVert#1\right\rVert}
\newcommand{\pd}[2]{\frac{\partial#1}{\partial#2}}
\newcommand{\Et}[1]{\E_{t\in\T}[#1]}
\newcommand{\Ex}[1]{\E_{(\bx,y)\sim\cD}[#1]}
\newcommand{\texp}[1]{\quad\text{\hlt{(#1)}}}

\DeclareMathOperator*{\argmin}{arg\,min}

\newtheorem{lemma}{Lemma}
\newtheorem{proposition}{Proposition}
\newtheorem{example}{Example}

\usepackage[pagebackref=true,breaklinks=true,colorlinks,bookmarks=false,hypertexnames=false]{hyperref}

\begin{document}

\title{More Is More - Narrowing the Generalization Gap by Adding Classification Heads}

\author{Roee Cates\\
The Hebrew University of Jerusalem\\
{\tt\small roee.cates@mail.huji.ac.il}
\and
Daphna Weinshall\\
The Hebrew University of Jerusalem\\
{\tt\small daphna@cs.huji.ac.il}
}

\date{}

\maketitle

\begin{abstract}
Overfit is a fundamental problem in machine learning in general, and in deep learning in particular. In order  to reduce overfit and improve generalization in the classification of images, some employ invariance to a group of transformations, such as rotations and reflections. However, since not all objects exhibit necessarily the same invariance, it seems desirable to allow the network to learn the useful level of invariance from the data. To this end, motivated by self-supervision, we introduce an architecture enhancement for existing neural network models based on input transformations, termed 'TransNet', together with a training algorithm suitable for it. Our model can be employed during training time only and then pruned for prediction, resulting in an equivalent architecture to the base model. Thus pruned, we show that our model improves performance on various data-sets while exhibiting improved generalization, which is achieved in turn by enforcing soft invariance on the convolutional kernels of the last layer in the base model. Theoretical analysis is provided to support the proposed method.
\end{abstract}

\section{Introduction}

Deep neural network models currently define the state of the art in many computer vision tasks, as well as speech recognition and other areas. These expressive models are able to model complicated input-output relations. At the same time, models of such large capacity are often prone to overfit, \ie performing significantly better on the training set as compared to the test set. This phenomenon is also called the \emph{generalization gap}.

We propose a method to narrow this generalization gap. Our model, which is called \emph{TransNet}, is defined by a set of input transformations. It augments an existing Convolutional Neural Network (CNN) architecture by allocating a specific head - a fully-connected layer which receives as input the penultimate layer of the base CNN - for each input transformation (see Fig.~\ref{fig:TransNet}). The transformations associated with the model's heads are not restricted apriori.

\begin{figure}[t]
\begin{center}
\includegraphics[width=0.9\linewidth]{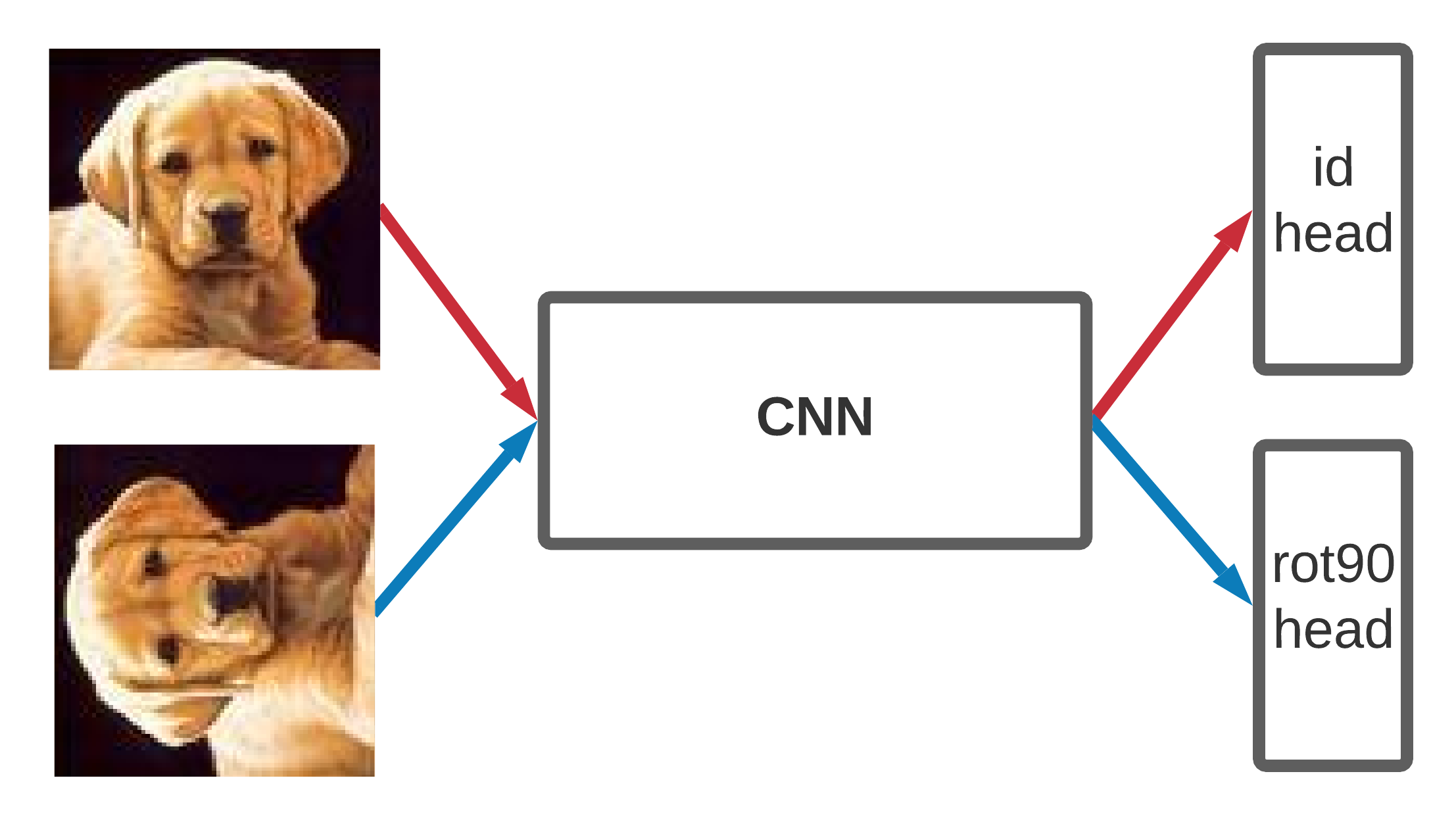}
\end{center}
\vspace{-0.35cm}
\caption{Illustration of the \emph{TransNet} architecture, which consists of 2 heads associated with 2 transformations, the identity and rotation by $90^{\circ}$. Each head classifies images transformed by associated transformation, while both share the same backbone.}
\vspace{-0.35cm}
\label{fig:TransNet}
\end{figure}

The idea behind the proposed architecture is that each head can specialize in a different yet related classification task. We note that any CNN model can be viewed as a special case of the \emph{TransNet} model, consisting of a single head associated with the identity transformation. The overall task is typically harder when training \emph{TransNet}, as compared to the base CNN architecture. Yet by training multiple heads, which share the convolutional backbone, we hope to reduce the model's overfit by providing a form of regularization.

In Section~\ref{sec:TransNet} we define the basic model and the training algorithm designed to train it (see Alg.~\ref{alg:1}). We then discuss the type of transformations that can be useful when learning to classify images. We also discuss the model's variations: (i) pruned version that employs multiple heads during training and then keeps only the head associated with the identity transformation for prediction; (ii) the full version where all heads are used in both training and prediction.

Theoretical investigation of this model is provided in Section~\ref{sec:theo_analysis}, using the dihedral group of transformations ($D_4$) that includes rotations by $90^o$ and reflections. We first prove that under certain mild assumptions, instead of applying each dihedral transformation to the input, one can compile it into the CNN model's weights by applying the inverse transformation to the convolutional kernels. In order to obtain intuition about the inductive bias of the model's training algorithm in complex realistic frameworks, we analyze the model's inductive bias using a simplified framework.

In Section~\ref{sec:exp_results} we describe our empirical results. We first introduce a novel invariance score ($IS$), designed to measure the model's kernel invariance under a given group of transformations. $IS$ effectively measures the inductive bias imposed on the model's weights by the training algorithm.  To achieve a fair comparison, we compare a regular CNN model traditionally trained, to the same model trained like a \emph{TransNet} model as follows: heads are added to the base model, it is trained as a \emph{TransNet} model, and then the extra heads are pruned. We then show that training as \emph{TransNet} improves test accuracy as compared to the base model. This improvement was achieved while keeping the optimized hyper-parameters of the base CNN model, suggesting that further improvement by fine tuning may be possible. We demonstrate the increased invariance of the model's kernels when trained with \emph{TransNet}.

\subsection*{Our Contribution}
\begin{itemize}[itemsep=1pt]
    \item Introduce \emph{TransNet} - a model inspired by self-supervision for supervised learning that imposes partial invariance to a group of transformations.
    \item Introduce an invariance score ($IS$) for CNN convolutional kernels.
    \item Theoretical investigation of the inductive bias implied by the \emph{TransNet} training algorithm.
    \item Demonstrate empirically how both the full and pruned versions of \emph{TransNet} improve accuracy.
\end{itemize}

\section{Related Work}

\myparagraph{Overfit.} A fundamental and long-standing issue in machine learning, overfit occurs when a learning algorithm minimizes the train loss, but generalizes poorly to the unseen test set. Many methods were developed to mitigate this problem, including  \emph{early stopping} - when training is halted as soon as the loss over a validation set starts to increase, and \emph{regularization} - when a penalty term is added to the optimization loss. Other related ideas, which achieve similar goals, include dropout \cite{srivastava2014dropout}, batch normalization \cite{ioffe2015batch}, transfer learning \cite{shao2014transfer, weiss2016survey}, and data augmentation \cite{cubuk2019autoaugment, zhong2020random}. 

\myparagraph{Self-Supervised Learning.}
A family of learning algorithms that train a model using self generated labels (\eg the orientation of an image), in order to exploit unlabeled data as well as extract more information from labeled data. Self training algorithms are used for representation learning, by training a deep network to solve pretext tasks where labels can be produced directly from the data. Such tasks include colorization \cite{zhang2016colorful, larsson2016learning}, placing image patches in the right place \cite{noroozi1603unsupervised, doersch2015unsupervised}, inpainting \cite{pathak2016context} and orientation prediction \cite{gidaris2018unsupervised}. Typically, self-supervision is used in unsupervised learning  \cite{dosovitskiy2015discriminative}, to impose some structure on the data, or in semi-supervised learning \cite{zhai2019s4l,hendrycks2019using}. Our work is motivated by \emph{RotNet}, an orientation prediction method suggested by \cite{gidaris2018unsupervised}. It differs from \cite{zhai2019s4l, hendrycks2019using}, as we allocate a specific classification head for each input transformation rather than predicting the self-supervised label with a separate head.

\myparagraph{Equivariant CNNs.} 
Many computer vision algorithms are designed to exhibit some form of invariance to a transformation of the input, including geometric transformations \cite{mundy1992geometric}, transformations of time \cite{turaga2009locally}, or changes in pose and illumination  \cite{paysan20093d}. Equivariance is a more relaxed property, exploited for example by CNN models when translation is concerned. Work on CNN models that enforces strict equivariance includes \cite{sifre2013rotation,gens2014deep,clark2015training,ngiam2010tiled,cohen2016group,dieleman2016exploiting}. Like these methods, our method seeks to achieve invariance by employing weight sharing of the convolution layers between multiple heads. But unlike these methods, the invariance constraint is soft. Soft equivariance is also seen in works like \cite{dieleman2015rotation}, which employs a convolutional layer that simultaneously feeds rotated and flipped versions of the original image to a CNN model, or \cite{wu2015flip} that appends rotation and reflection versions of each convolutional kernel.

\section{TransNet}
\label{sec:TransNet}

\myparagraph{Notations and definitions}
Let $\X = \{(\bx_i, y_i)\}_{i=1}^n$ denote the training data, where $\bx_i\in\R^d$ denotes the i-th data point and $y_i\in[K]$ its corresponding label. Let $\cD$ denote the data distribution from which the samples are drawn. Let $\cH$ denote the set of hypotheses, where $h_\bgt \in \cH$ is defined by its parameters $\bgt$ (often we use $h=h_\bgt$ to simplify notations). Let $\ell(h,\bx,y)$ denote the loss of hypothesis $h$ when given sample $(\bx,y)$. The overall loss is:
\begin{equation}
    \cL(h,\X) = \Ex{\ell(h,\bx,y)}
\end{equation}
Our objective is to find the optimal hypothesis:
\begin{equation}
    h^* := \argmin_{h\in \cH} \cL(h,\X)
\end{equation}

For simplicity, whenever the underlying distribution of a random variable isn't explicitly defined we use the uniform distribution, \eg $\E_{a\in\A}[a] = 1/|\A|\sum_{i=1}^{|\A|}a$.

\subsection{Model architecture}

The \emph{TransNet} architecture is defined by a set of input transformations $\T=\{t_j\}_{j=1}^m$, where each transformation $t\in\T$ operates on the inputs ($t: \R^d \rightarrow \R^d$) and is associated with a corresponding model's head. Thus each transformation operates on datapoint $\bx$ as $t(\bx)$, and the transformed data-set is defined as:
\begin{equation}
    t(\X) := \{(t(\bx_i), y_i)\}_{i=1}^n
\end{equation}

Given an existing NN model $h$, henceforth called the \emph{base model}, we can split it to two components: all the layers except for the last one denoted $f$, and the last layer $g$ assumed to be a fully-connected layer. Thus $h = g \circ f$. Next, we enhance model $h$ by replacing $g$ with $|\T| = m$ heads, where each head is an independent fully connected layer  $g_t$ associated with a specific transformation $t\in\T$. Formally, each head is defined by $h_t=g_t \circ f$, and it operates on the corresponding transformed input as $h_t(t(\bx))$. 

The full model, with its $m$ heads, is denoted by $h_\T:=\{h_t\}_{t\in\T}$, and operates on the input as follows:
\begin{eqnarray*}
h_\T(\bx):=\E_{t\in\T }[h_t(t(\bx))]
\end{eqnarray*}
The corresponding loss of the full model is defined as:
\begin{equation}
    \cL_\T(h_\T,\X) := \Et{\cL(h_t,t(\X))}
\label{eq:trans_loss}
\end{equation}
Note that the resulting model (see Fig.~\ref{fig:TransNet}) essentially represents $m$ models, which share via $f$ all the weights up to the last fully-connected layer. Each of these models can be used separately, as we do later on.

\subsection{Training algorithm}

Our method uses SGD with a few modifications to minimize the transformation loss (\ref{eq:trans_loss}), as detailed in Alg.~\ref{alg:1}. Relying on the fact that each batch is sampled i.i.d. from $\cD$, we can prove (see Lemma~\ref{lemma:unbiased}) the desirable property that the sampled loss $\cL_\T(h_\T,\B)$ is an unbiased estimator for the transformation loss $\cL_\T(h_\T,\X)$. This justifies the use of Alg.~\ref{alg:1} to optimize the transformation loss.

\IncMargin{1em}
\begin{algorithm}[h]
\SetAlgoLined
\caption{Training the \emph{TransNet} model}
\label{alg:1}
    \SetKwInOut{Input}{input}
    \SetKwInOut{Output}{output}
    
    \Input{\emph{TransNet} model $h_\T$, batch size $b$, maximum iterations num $MAX\_ITER$}
    \Output{trained \emph{TransNet} model}
    \BlankLine
    
    \For{$i=1 \dots MAX\_ITER$}{
    sample a batch $\B=\{(\bx_k,y_k)\}_{k=1}^b \overset{iid}{\sim} \cD^b$\\
    \hlt{forward:} \\
    \For{$t\in\T$}{
        $\cL(h_t,\B)=\frac{1}{b}\sum_{k=1}^b\ell(h_t,t(\bx_k),y_k)$ \nllabel{alg:1:mod}
    }
    $\cL_\T(h_\T,\B)=\frac{1}{m}\sum_{t\in\T}\cL(h_t,\B)$\\
    \hlt{backward (SGD):} \\
    update the model's weights by differentiating the sampled  loss $\cL_\T(h_\T,\B)$
    }
\end{algorithm}
\DecMargin{1em}


\begin{lemma}
\label{lemma:unbiased}
Given batch $\B$, the sampled transformation loss $\cL_\T(h_\T,\B)$ is an unbiased estimator for the transformation loss $\cL_\T(h_\T,\X)$.
\end{lemma}
\begin{proof}
\begin{equation}
\begin{split}
    \E&_{\B\sim\cD^b}[\cL_\T(h_\T,\B)]\\
    &= \E_{\B\sim\cD^b}[\Et{\cL(h_t,t(\B))}]\\
    &= \Et{\E_{\B\sim\cD^b}[\cL(h_t,t(\B))]} \texp{$\B\overset{iid}{\sim}\cD^b$}\\
    &= \Et{\cL(h_t,t(\X))}\\
    &= \cL_\T(h_\T,\X)
\end{split}
\end{equation}
\end{proof}

\subsection{Transformations}

\myparagraph{Which transformations should we use?}
Given a specific data-set, we distinguish between transformations that occur naturally in the data-set versus such transformations that do not. For example, horizontal flip can naturally occur in the CIFAR-10 data-set, but not in the MNIST data-set. \emph{TransNet} can only benefit from transformations that do not occur naturally in the target data-set, in order for each head to learn a well defined and non-overlapping classification task. Transformations that occur naturally in the data-set are often used for data augmentation, as by definition they do not change the data domain.

\myparagraph{Dihedral group $D_4$.}
As mentioned earlier, the \emph{TransNet} model is defined by a set of input transformations $\T$. We constrain $\T$ to be a subset of the dihedral group $D_4$, which includes reflections and rotations by multiplications of $90^\circ$. We denote a horizontal reflection by $m$ and a counter-clockwise $90^\circ$ rotation by $r$. Using these two elements we can express all the $D_4$ group elements as $\{r^i, m \circ r^i \ | \ i\in {0,1,2,3}\}$. These transformations were chosen because, as mentioned in \cite{gidaris2018unsupervised}, their application is relatively efficient and does not leave artifacts in the image (unlike scaling or change of aspect ratio).

Note that these transformations can be applied to any 3D tensor while operating on the height and width dimensions, including an input image as well as the model's kernels. When applying a transformation $t$ to the model's weights $\bgt$, denoted $t(\bgt)$, the notation implies that $t$ operates on the model's kernels separately, not affecting other layers such as the fully-connected ones (see Fig.~\ref{fig:comp_tr_0}).

\begin{figure}[ht]
\begin{center}
\includegraphics[width=0.9\linewidth]{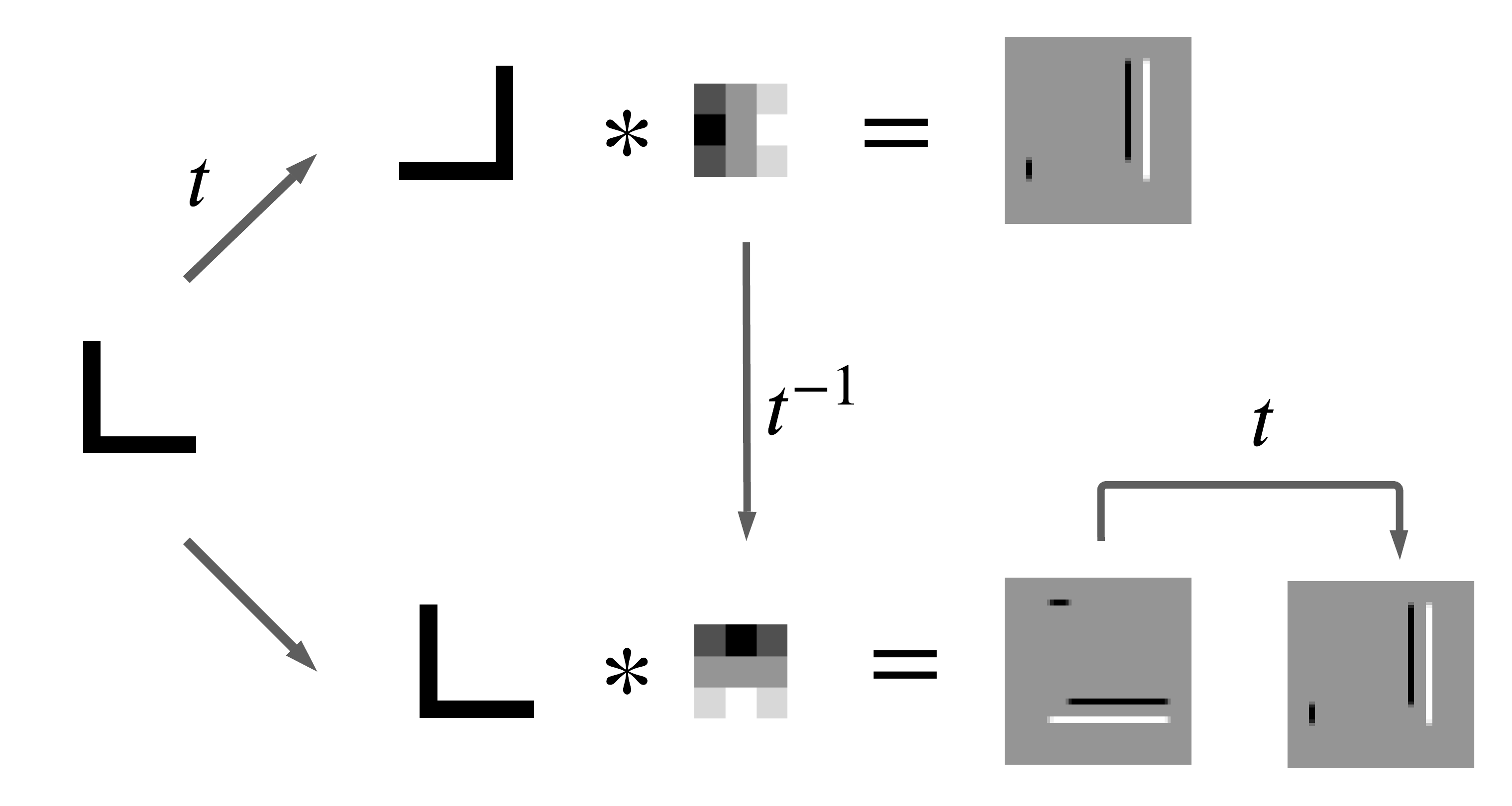}
\end{center}
\vspace{-0.35cm}
\caption{The transformed input convolved with a kernel (upper path) equals to the transformation applied on the output of the input convolved with the inversely transformed kernel (lower path).}
\vspace{-0.35cm}
\label{fig:comp_tr_0}
\end{figure}

\subsection{Model variations}

Once trained, the full \emph{TransNet} model can be viewed as an ensemble of $m$ shared classifiers. Its time complexity is linear with the number of heads, almost equivalent to an ensemble of the base CNN model, since the time needed to apply each one of the $D_4$ transformations to the input is negligible as compared to the time needed for the model to process the input. Differently, the space complexity is almost equivalent to the space complexity of only one base CNN model\footnote{\label{fn:space}Each additional head adds 102K ($\sim$0.45\%) and 513K ($\sim$0.90\%) extra parameters to the basic ResNet18 model when training  CIFAR-100 and ImageNet-200 respectively.}.

We note that one can prune each one of the model's heads, thus leaving a smaller ensemble of up to $m$ classifiers.
A useful reduction prunes all the model's heads except one, typically the one corresponding to the identity transformation, which yields a regular CNN that is equivalent in terms of time and space complexity to the base architecture used to build the \emph{TransNet} model. Having done so, we can evaluate the effect of the \emph{TransNet} architecture's and its training algorithm's inductive bias solely on the training procedure, by comparing the pruned \emph{TransNet} to the base CNN model (see Section~\ref{sec:exp_results}).

\section{Theoretical Analysis} \label{sec:theo_analysis}

In this section we analyze theoretically the \emph{TransNet} model. We consider the following basic CNN architecture:
\begin{equation}
    h_\bgt = g \circ l_{inv} \circ \prod_{i=1}^kc_i
    \label{eq:h_theta}
\end{equation}
where $g$ denotes a fully-connected layer, $l_{inv}$ denotes an invariant layer under the $D_4$ transformations group (\eg a global average pooling layer - GAP), and $\{c_i\}_{i\in[k]}$ denote convolutional layers\footnote{While each convolutional layer may be followed by  ReLU and Batch Normalization \cite{ioffe2015batch} layers, this doesn't change the analysis so we obviate the extra notation.}. The \emph{TransNet} model extends the basic model by appending additional heads:
\begin{equation}
    h_{\T,\bgt} = \{g_t \circ l_{inv} \circ \prod_{i=1}^kc_i\}_{t\in\T}
\end{equation}
We denote the parameters of a fully-connected or a convolutional layer by subscripts of $w$ (weight) and $b$ (bias), \eg $g(\bx)=g_w \cdot \bx + g_b$.

\subsection{Transformation compilation}

Transformations in the dihedral $D_4$ group satisfy another important property, expressed by the following proposition:
\begin{proposition}
\label{prop:comp_tr}
    Let $h_\bgt$ denote a CNN model where the last convolutional layer is followed by an invariant layer under the $D_4$ group. Then any transformation $t\in D_4$ applied to the input image can be compiled into the model's weights $\bgt$ as follows:
    \begin{equation}
        \forall t\in D_4 \quad \forall \bx \in \X: \quad h_\bgt(t(\bx)) = h_{t^{-1}(\bgt)}(\bx)
    \end{equation}
\begin{proof}
    By induction on $k$ we can show that:
    \begin{equation}
        \prod_{i=1}^kc_i\circ t(\bx) = t \circ \prod_{i=1}^kt^{-1}(c_i)(\bx)
        \label{eq:induction}
    \end{equation}
     (see Fig.~\ref{fig:comp_tr_0}). Plugging (\ref{eq:induction}) into (\ref{eq:h_theta}), we get:
    \begin{equation*}
    \begin{split}
        h_\bgt(t(\bx)) &= g \circ l_{inv} \circ \prod_{i=1}^kc_i \circ t(\bx)\\
        &= g \circ l_{inv} \circ t \circ \prod_{i=1}^kt^{-1}(c_i)(\bx)\\
        &= g \circ l_{inv} \circ \prod_{i=1}^kt^{-1}(c_i)(\bx) \texp{$l_{inv}\circ t = l_{inv}$}\\
        &= h_{t^{-1}(\bgt)}(\bx)
    \end{split}
    \end{equation*}
\end{proof}
\end{proposition}

\myparagraph{Implication.}
The ResNet model \cite{he2016deep} used in our experiments satisfies the pre-condition in the proposition stated above, since it contains a GAP layer \cite{lin2013network} after the last convolutional layer, and GAP is invariant under $D_4$.

\subsection{Single vs. multiple headed model}
\label{sec:single_vs_multi_headed}

In order to acquire intuition regarding the inductive bias implied by training algorithm Alg.~\ref{alg:1}, we consider two cases, a single and a double headed model, trained with the same training algorithm. A single headed model is a special case of the full multi-headed model, where all the heads share weights $h_t(t(\bx))=h(t(\bx))~\forall t$, and the loss in line~\ref{alg:1:mod} of Alg.~\ref{alg:1} becomes $\cL(h,\B)=\frac{1}{b}\sum_{k=1}^b\ell(h,t(\bx_k),y_k)$.

As it's hard to analyze non-convex deep neural networks, we focus on a simplified framework and consider a convex optimization problem where the loss function is convex \wrt the model's parameters $\bgt$. We also assume that the model's transformations in $\T$ form a group\footnote{$\T$ being a group is a technical constraint needed for the analysis, not required by the algorithm.}.

\myparagraph{Single Headed model Analysis.}
In this simplified case, we can prove the following strict proposition:
\begin{proposition}
\label{prop:inv_kernel}
    Let  $h_\bgt$ denote a CNN model satisfying the pre-condition of Prop.~\ref{prop:comp_tr}, and $\T \subset D_4$ a transformations group. Then the optimal transformation loss $\cL_\T$ (see Eq.~\ref{eq:trans_loss}) is obtained by invariant model's weights under the transformations $\T$. Formally:
    \begin{equation*}
        \exists \bgt_0: (\forall t \in \T: \bgt_0=t(\bgt_0)) \wedge (\bgt_0 \in \argmin_\bgt \cL_\T(\bgt,\X))
    \end{equation*}
\end{proposition}
\begin{proof}
To simplify the notations, henceforth we let $\bgt$ denote the model $h_\bgt$.
\begin{equation*}
\begin{split}
    \cL_\T&(\bgt,\X)\\
    &= \Et{\cL(\bgt,t(\X))}\\
    &= \Et{\Ex{\ell(\bgt,t(\bx),y)}}\\
    &= \Et{\Ex{\ell(t^{-1}(\bgt),\bx,y)}} \texp{by Prop.~\ref{prop:comp_tr}}\\
    &= \Ex{\Et{\ell(t^{-1}(\bgt),\bx,y)}}\\
    &\geq \Ex{\ell(\Et{t^{-1}(\bgt)},\bx,y)} \texp{Jensen's inequality}\\
    &= \Ex{\ell(\bar{\bgt},\bx,y)} \texp{$\bar{\bgt}:=\Et{t(\bgt))}, \quad \T=\T^{-1}$}\\
    &= \cL(\bar{\bgt},\X)\\
    &= \Et{\cL(t^{-1}(\bar{\bgt}),\X)} \texp{$\bar{\bgt}$ is invariant under $\T$}\\
    &= \Et{\cL(\bar{\bgt},t(\X))} \texp{by Prop.~\ref{prop:comp_tr}}\\
    &= \cL_\T(\bar{\bgt},\X)
\end{split}
\end{equation*}
Above we use the fact that $\bar{\bgt}$ is invariant under $\T$ since $\T$ is a group and thus $t_0\T=\T$, hence:
\begin{equation*}
    t_0(\bar{\bgt}) = t_0(\Et{t(\bgt)}) = \Et{t_0 \circ t(\bgt)} = \Et{t(\bgt)} = \bar{\bgt}
\end{equation*}
\end{proof}

\myparagraph{Double headed model.}
In light of Prop.~\ref{prop:inv_kernel} we now present a counter example, which shows that Prop.~\ref{prop:inv_kernel} isn't true for the general \emph{TransNet} model.
\begin{example}
    Let $\T=\{t_1=r^0, t_2=m\circ r^2\} \subset D_4$ denote the transformations group consisting of the identity and the vertical reflection transformations. Let $h_{\T,\bgt} = \{h_i=g_i \circ GAP \circ c\}_{i=1}^2$ denote a double headed \emph{TransNet} model, which comprises a single convolutional layer (1 channel in and 2 channels out), followed by a GAP layer and then 2 fully-connected layers $\{g_i\}_{i=1}^2$, one for each head. Each $g_i$ outputs a vector of size 2. The data-set $\X=\{(\bx_1,y_1),(\bx_2,y_2)\}$ consists of 2 examples:
    \begin{equation*}
    \bx_1=\begin{bmatrix}
    1 & 1 & 1\\
    0 & 0 & 0\\
    0 & 0 & 0
    \end{bmatrix}, y_1 = 1, \quad
    \bx_2=\begin{bmatrix}
    0 & 0 & 0\\
    0 & 0 & 0\\
    1 & 1 & 1
    \end{bmatrix}, y_2 = 2
   \end{equation*}
    Note that $\bx_2=t_2(\bx_1)$\footnote{This example may seem rather artificial, but in fact this isn't such a rare case. \Eg, the airplane and the ship classes, both found in the CIFAR-10 data-set, that share similar blue background.}.
    
    Now, assume the model's convolutional layer $c$ is composed of 2 invariant kernels under $\T$, and denote it by $c_{inv}$. Let $i\in{1,2}$, then:
    \begin{equation}
    \begin{split}
        h_i(\bx_2) &= h_i(t_2(\bx_1)) = g_i \circ GAP \circ c_{inv} \circ t_2 (\bx_1)\\ 
        &= g_i \circ GAP \circ c_{inv} (\bx_1) = h_i(\bx_1)
    \end{split}
    \end{equation}
    In this case both heads predict the same output for both inputs with different labels, thus:
    \begin{equation*}
    \cL(h_i,t_i(\X))>0 \implies \cL_\T(h_{\T,\bgt},\X)>0
    \end{equation*}
    In contrast, by setting $c_w=(\bx_1,\bx_2), c_b=(0,0)$, which isn't invariant under $\T$, as well as:
    \begin{equation*}
    g_{1,w}=\begin{bmatrix}1&0\\0&1\end{bmatrix}, g_{1,b}=\begin{bmatrix}0\\0\end{bmatrix} \quad
    g_{2,w}=\begin{bmatrix}0&1\\1&0\end{bmatrix}, g_{2,b}=\begin{bmatrix}0\\0\end{bmatrix},
    \end{equation*}
    we obtain:
    \begin{equation*}
    \cL(h_i,t_i(\X))=0 \implies \cL_\T(h_{\T,\bgt},\X)=0.
    \end{equation*}
    We may conclude that the optimal model's kernels aren't invariant under $\T$, as opposed to the claim of Prop.~\ref{prop:inv_kernel}.
\end{example}

\myparagraph{Discussion.}
The intuition we derive from the analysis above is that the training algorithm (Alg.~\ref{alg:1}) implies an invariant inductive bias on the model's kernels as proved in the single headed model, while not strictly enforcing invariance as shown by the counter example of the double headed model.

\section{Experimental Results} 
\label{sec:exp_results}

\myparagraph{data-sets.}
For evaluation we used the 5 image classification data-sets detailed in Table~\ref{table:data-sets}. These diverse data-sets allow us to evaluate our method across different image resolutions and number of predicted classes.

\begin{table}[h]
\begin{center}
\begin{tabular}{lccc}
\hline
{\bf Name} & {\bf Classes} & {\bf Train/Test} & {\bf dim}\\
 & & {\bf Samples} &\\
\hline 
CIFAR-10 \cite{krizhevsky2009learning}  & 10    & 50K/10K   & 32 \\	
CIFAR-100 \cite{krizhevsky2009learning} & 100   & 50K/10K   & 32 \\	
ImageNette \cite{imagewang}             & 10    & 10K/4K    & 224 \\	
ImageWoof \cite{imagewang}              & 10    & 10K/4K    & 224 \\	
ImageNet-200                            & 200   & 260K/10K  & 224 \\	
\hline
\end{tabular}
\caption{The data-sets used in our experiments. The dimension of each example, a color image, is \textbf{dim}$\times$\textbf{dim}$\times 3$ pixels. ImageNette represents 10 easy to classify classes from ImageNet \cite{deng2009ImageNet}, while ImageWoof represents 10 hard to classify classes of dog breeds from ImageNet. ImageNet-200 represents 200 classes from ImageNet (same classes as in \cite{le2015tiny}) of full size images.}
\label{table:data-sets}
\end{center}
\vspace{-0.5cm}
\end{table}

\begin{table*}[th]
\begin{center}
\begin{tabular}{lccccc}
{\bf MODEL} & {\bf CIFAR-10} & {\bf CIFAR-100} & {\bf ImageNette} & {\bf ImageWoof} & {\bf ImageNet-200} \\ \hline
base-CNN     & 95.57 $\pm$ 0.08 & 76.56 $\pm$ 0.16 & 92.97 $\pm$ 0.16 & 87.27 $\pm$ 0.15 & 84.39 $\pm$ 0.07 \\ \hline
PT2-CNN & {\bf 95.99} $\pm$ 0.07 & {\bf 79.33} $\pm$ 0.15 & 93.84 $\pm$ 0.14 & {\bf 88.09} $\pm$ 0.30 & {\bf 85.17} $\pm$ 0.10 \\ \hline
PT3-CNN & 95.87 $\pm$ 0.04 & 79.08 $\pm$ 0.06 & {\bf 94.15} $\pm$ 0.16 & 87.79 $\pm$ 0.11 & 84.97 $\pm$ 0.95 \\ \hline
PT4-CNN & 95.73 $\pm$ 0.05 & 77.98 $\pm$ 0.17 & 93.94 $\pm$ 0.06 & 85.81 $\pm$ 0.79 & 84.02 $\pm$ 0.71 \\ \hline
\end{tabular}
\vspace{-0.5cm}
\end{center}
\caption{Accuracy of models with the same space and time complexity, comparing the Base CNN with pruned \emph{TransNet} models "PT$m$-CNN", where $m=2,3,4$ denotes the number of heads in training. Mean and standard error for 3 repetitions are shown.}
\label{table:acc_same_complexity}
\end{table*}

\begin{table*}[th]
\begin{center}
\begin{tabular}{lccccc}
\multicolumn{1}{l}{\bf MODEL} & {\bf CIFAR-10} & {\bf CIFAR-100} & {\bf ImageNette} & {\bf ImageWoof} & {\bf ImageNet-200} \\ \hline 
base-CNN     & 95.57 $\pm$ 0.08 & 76.56 $\pm$ 0.16 & 92.97 $\pm$ 0.16 & 87.27 $\pm$ 0.15 & 84.39 $\pm$ 0.07 \\ \hline
T2-CNN  & 96.22 $\pm$ 0.10 & 80.35 $\pm$ 0.06 & 94.02 $\pm$ 0.13 & 88.36 $\pm$ 0.33 & 85.47 $\pm$ 0.14 \\ \hline
T3-CNN  & {\bf 96.33} $\pm$ 0.06 & {\bf 80.92} $\pm$ 0.08 & 94.39 $\pm$ 0.07 & {\bf 88.79} $\pm$ 0.25 & {\bf 85.68} $\pm$ 0.20 \\ \hline
T4-CNN  & 96.17 $\pm$ 0.01 & 79.94 $\pm$ 0.16 & {\bf 94.67} $\pm$ 0.06 & 87.05 $\pm$ 0.75 & 85.54 $\pm$ 0.11 \\ \hline
\end{tabular}
\vspace{-0.5cm}
\end{center}
\caption{Accuracy of models with similar space complexity and different time complexity, comparing the Base CNN with full \emph{TransNet} models. With $m$ denoting the number of heads, chosen to be 2,3 or 4, the prediction time complexity of the respective \emph{TransNet} model "T$m$-CNN" is $m$ times larger than the base CNN. Mean and standard error for 3 repetitions are shown. }
\label{table:acc_diff_time_complexity}
\end{table*}

\myparagraph{Implementation Details.}
We employed the ResNet18 \cite{he2016deep} architecture for all the data-sets except for ImageNet-200, which was evaluated using the ResNet50 architecture  (see Appendix~\ref{app:imp_details} for more implementation details).

\myparagraph{Notations.}
\begin{itemize}[itemsep=1pt]
    \item "base CNN" - a regular convolutional neural network, identical to the \emph{TransNet} model with only the head corresponding to the identity transformation. 
    \item "PT$m$-CNN" - a pruned \emph{TransNet} model trained with $m$ heads, where a single head is left and used for prediction\footnote{In our experiments we chose the head associated with the identity ($r^0$) transformation when evaluating a pruned \emph{TransNet}. Note, however, that we could have chosen the best head in terms of accuracy, as it follows from Prop.~\ref{prop:comp_tr} that its transformation can be compiled into the model's weights.}. It has the same space and time complexity as the base CNN.
    \item "T$m$-CNN" - a full \emph{TransNet} model trained with $m$ heads, where all are used for prediction. It has roughly the same space complexity\textsuperscript{\ref{fn:space}} and $m$ times the time complexity as compared to the base CNN.
\end{itemize}
To denote an ensemble of the models above, we add a suffix of a number in parentheses, \eg T2-CNN (3) is an ensemble of 3 T2-CNN models.

\subsection{Models accuracy, comparative results}
\label{sec:accuracy}

We now compare the accuracy of the "base-CNN", "PT$m$-CNN" and "T$m$-CNN" models, where $m=2,3,4$ denotes the number of heads of the \emph{TransNet} model, and their ensembles, across all the data-sets listed in Table~\ref{table:data-sets}.

\myparagraph{Models with the same space and time complexity.}
First, we evaluate the pruned \emph{TransNet} model by comparing the "PT$m$-CNN" models with the "base-CNN" model, see Table~\ref{table:acc_same_complexity}. Essentially, we evaluate the effect of using the \emph{TransNet} model only for training, as the final "PT$m$-CNN" models are identical to the "base-CNN" model regardless of $m$. We can clearly see the inductive bias implied by the training procedure. We also see that \emph{TransNet} training improves the accuracy of the final "base-CNN" classifier across all the evaluated data-sets.

\myparagraph{Models with similar space complexity, different time complexity.}
Next, we evaluate the full \emph{TransNet} model by comparing the "T$m$-CNN" models with the "base-CNN" model, see Table~\ref{table:acc_diff_time_complexity}. Despite the fact that the full \emph{TransNet} model processes the (transformed) input $m$ times more as compared to the "base-CNN" model, its architecture is not significantly larger than the base-CNN's. The full \emph{TransNet} adds to the "base-CNN" a negligible number of parameters, in the form of its multiple heads\textsuperscript{\ref{fn:space}}. Clearly the full \emph{TransNet} model improves the accuracy as compared to the "base-CNN" model, and also as compared to the pruned \emph{TransNet} model. Thus, if the additional runtime complexity during test is not an issue, it is beneficial to employ the full \emph{TransNet} model during test time. In fact, one can process the input image once, and then choose whether to continue processing it with the other heads to improve the prediction, all this while keeping roughly the same space complexity.

\myparagraph{Ensembles: models with similar time complexity, different space complexity.}
Here we evaluate ensembles of pruned \emph{TransNet} models, and compare them to a single full \emph{TransNet} model that can be seen as a space-efficient ensemble: full \emph{TransNet} generates $m$ predictions with only $~1/m$ parameters, where $m$ is the number of \emph{TransNet} heads. Results are shown in Fig.~\ref{fig:ens_acc}. Clearly an ensemble of pruned \emph{TransNet} models is superior to an ensemble of base CNN models, suggesting that the accuracy gain achieved by the pruned \emph{TransNet} model doesn't overlap with the accuracy gain achieved by using an ensemble of classifiers. Furthermore, we observe that the full \emph{TransNet} model exhibits competitive accuracy results, with 2 and 3 heads, as compared to an ensemble of 2 or 3 base CNN models respectively. This is achieved while utilizing $~1/2$ and $~1/3$ as many parameters respectively.

\begin{figure}[h]
\begin{center}
\includegraphics[width=0.9\linewidth]{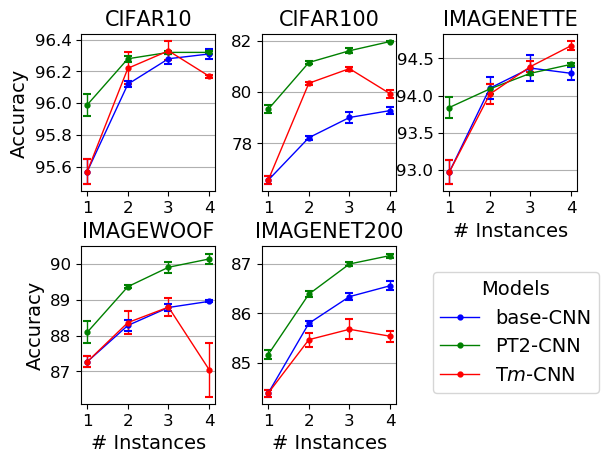}
\end{center}
\vspace{-0.35cm}
\caption{Model accuracy as a function of the number of instances ($X$-axis) processed during prediction. Each instance requires a complete run from input to output. An ensemble includes: $m$ independent base CNN classifiers for "CNN"; $m$ pruned \emph{TransNet} trained with 2 heads for "PT2-CNN"; and one \emph{TransNet} model with $m$ heads, where $m$ is the ensemble size, for "T$m$-CNN".}
\label{fig:ens_acc}
\end{figure}

\myparagraph{Accuracy vs. generalization.}
In Fig.~\ref{fig:ens_acc} we can see that 2 heads improve the model's performance across all data-sets, 3 heads improve it on most of the data-sets, and 4 heads actually reduce performance on most data-sets. We hypothesize that too many heads impose too strict an inductive bias on the model's kernels. Thus, although generalization is improved, test accuracy is reduced due to insufficient variance. Further analysis is presented in the next section.

\subsection{Generalization}
\label{sec:generalization}

We've seen in Section~\ref{sec:accuracy} that the \emph{TransNet} model, whether full or pruned, achieves better test accuracy as compared to the base CNN model. This occurs despite the fact that the transformation loss $\cL_\T(h_\T,\X)$ minimized by the \emph{TransNet} model is more demanding than the loss $\cL(h,\X)$ minimized by the base CNN, and appears harder to optimize. This conjecture is justified by the following Lemma:
\begin{lemma}
    Let $h_\T$ denote a \emph{TransNet} model that obtains transformation loss of $a:=\cL_\T(h_\T,\X)$. Then there exists a reduction from $h_\T$ to the base CNN model $h$ that obtains a loss of at most $a$, \ie $\cL(h,\X) \leq a$.
\end{lemma}
\begin{proof}
    $a=\cL_\T(h_\T,\X)=\Et{\cL(h_{\bgt_t},t(\X))}$, so there must be a transformation $t\in\T$ s.t. $\cL(h_{\bgt_t},t(\X))\leq a$. Now, one can compile the transformation $t$ into $h_{\bgt_t}$ (see Prop.~\ref{prop:comp_tr}) and get a base CNN: $\tilde{h}=h_{t^{-1}(\bgt_t)}$ which obtains $\cL(\tilde{h},\X) = \cL(h_{t^{-1}(\bgt_t)}, t(\X)) = \cL(h_{\bgt_t},t(\X)) \leq a$.
\end{proof}

Why is it, then, that the \emph{TransNet} model achieves overall better accuracy than the base CNN? The answer lies in its ability to achieve a better generalization. 

In order to measure the generalization capability of a model \wrt a data-set, we use the ratio between the test-set and train-set loss, where a lower ratio indicates better generalization. As illustrated in Fig.~\ref{fig:CIFAR100_gen}, clearly the pruned \emph{TransNet} models exhibit better generalization when compared to the base CNN model. Furthermore, the generalization improvement increases with the number of \emph{TransNet} model heads, which are only used for training and then pruned. The observed narrowing of the generalization gap occurs because, although the \emph{TransNet} model slightly increases the training loss, it more significantly decreases the test loss as compared to the base CNN.

\begin{figure}[h]
\begin{center}
\includegraphics[width=1.0\linewidth]{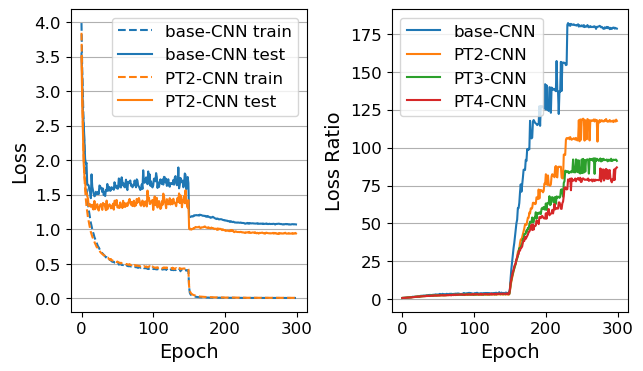}
\end{center}
\vspace{-0.35cm}
\caption{CIFAR-100 results. Left panel: learning curve of the Base CNN model ("base-CNN") and a pruned \emph{TransNet} model ("PT2-CNN"). Right panel: generalization score, test-train loss ratio, measured for the base-CNN model and various pruned \emph{TransNet} models with a different number of heads.}
\label{fig:CIFAR100_gen}
\end{figure}

We note that better generalization does not necessarily imply a better model. The "PT4-CNN" model generalizes better than any other model (see right panel of Fig.~\ref{fig:CIFAR100_gen}), but its test accuracy is lower as seen in  Table~\ref{table:acc_same_complexity}.

\subsection{Kernel invariance}

What characterizes the beneficial inductive bias implied by the \emph{TransNet} model and its training algorithm Alg.~\ref{alg:1}?. To answer this question, we investigate the emerging invariance of kernels in the convolutional layers of the learned network, \wrt the \emph{TransNet}  transformations set $\T$.

We start by introducing the "Invariance Score" ($IS$), which measures how invariant a 3D tensor is \wrt a transformations group. Specifically, given a convolutional kernel denoted by $\bw$ (3D tensor) and a set of transformations group $\T$, the $IS$ score is defined as follows:
\begin{equation}
    IS(\bw,\T) := \min_{\bu \in INV_\T}\norm{\bw-\bu}
\end{equation}
where $INV_\T$ is the set of invariant kernels (same shape as $\bw$) under $\T$, \ie 
$INV_\T:=\{\bu: \bu = t(\bu) \ \forall t\in\T \}$.

\begin{lemma}
\label{lemma:inv_kernel}
    $\argmin_{\bu\in INV_\T}\norm{\bw-\bu} = \Et{t(\bw)}$
\end{lemma}
\begin{proof}
    Let $\bu$ be an invariant tensor under $\T$.
    Define $f(\bu):= \norm{\bw-\bu}^2$. Note that $\argmin_{u\in INV_\T}\norm{\bw-\bu} = \argmin_{u\in INV_\T}f(u)$.
    
\begin{equation*}
\begin{split}
    f(\bu) &= \norm{\bw-\bu}^2\\
    &= \Et{\norm{\bw-t(\bu)}^2} \texp{$\bu$ is invariant under $\T$}\\
    &= \Et{\norm{t^{-1}(\bw)-\bu}^2}\\
    &= \Et{\norm{t(\bw)-\bu}^2} \texp{$\T=\T^{-1}$}\\
    &= \Et{\sum_{i=1}^{size(\bw)}(t(\bw)_i-\bu_i)^2}
\end{split}
\end{equation*}
Where index $i$ runs over all the tensors' elements. Finally, we differentiate $f$ to obtain its minimum:
\begin{equation*}
\begin{split}
    &\pd{f}{\bu_i} = \Et{-2(t(\bw)_i-\bu_i)} = 0\\
    &\implies \bu_i = \Et{[t(\bw)_i} \implies \bu = \Et{t(\bw)}\qedhere
\end{split}
\end{equation*}
\end{proof}
\noindent
Lemma~\ref{lemma:inv_kernel} gives a closed-form expression for the $IS$ gauge:
\begin{equation}
    IS(\bw,\T) = \norm{\bw-\Et{t(\bw)}}
\end{equation}

Equipped with this gauge, we can inspect the invariance level of the model's kernels \wrt a transformations group. Note that this measure allows us to compare the full \emph{TransNet} model with the base CNN model, as both share the same convolution layers. Since the transformations of the \emph{TransNet} model don't necessarily form a group, we use the minimal group containing these transformations - the group of all rotations  $\{r^i\}_{i=1}^4$.

In Fig.~\ref{fig:CIFAR100_IS_layers_graph} we can see that the full \emph{TransNet} model "T2-CNN" and the base CNN model demonstrate similar invariance level in all the convolutional layers but the last one. In Fig.~\ref{fig:CIFAR100_IS_layer17}, where the distribution of the $IS$ score over the last layer of 4 different models is fully shown, we can more clearly see that the last convolutional layer of full \emph{TransNet} models exhibits much higher invariance level as compared to the base CNN. This phenomenon is robust to the metric used in the $IS$ definition, with similar results when using "Pearson Correlation" or "Cosine Similarity". The increased invariance in the last convolutional layer is monotonically increasing with the number of heads in the \emph{TransNet} model, which is consistent with the generalization capability of these models (see Fig~\ref{fig:CIFAR100_gen}).

\begin{figure}[h]
\begin{center}
\includegraphics[width=0.9\linewidth]{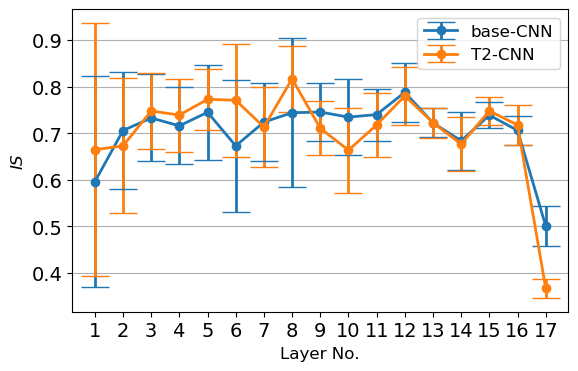}
\end{center}
\vspace{-0.35cm}
\caption{CIFAR-100 results, plotting the distribution of the $IS$ scores (mean and std) for the kernels in each layer of the different models. Invariance is measured \wrt the group of $90^\circ$ rotations.}
\label{fig:CIFAR100_IS_layers_graph}
\end{figure}

\begin{figure}[h]
\begin{center}
\includegraphics[width=0.9\linewidth]{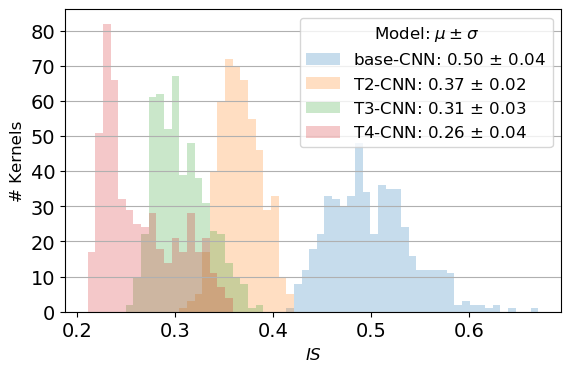}
\end{center}
\vspace{-0.35cm}
\caption{CIFAR-100 results, plotting the full distribution of the $IS$ scores for the kernels in the last (17-th) layer of the different models. Invariance is measured \wrt the group of $90^\circ$ rotations.}
\vspace{-0.35cm}
\label{fig:CIFAR100_IS_layer17}
\end{figure}

The generalization improvement achieved by the \emph{TransNet} model, as reported in Section~\ref{sec:generalization}, may be explained by this increased level of invariance, as highly invariant kernels have fewer degrees of freedom, and should therefore be less prone to overfit.
 
\subsection{Ablation Study}

\begin{table*}[t]
\begin{center}
\begin{tabular}{lccccc}
{\bf MODEL} & {\bf CIFAR-10} & {\bf CIFAR-100} & {\bf ImageNette} & {\bf ImageWoof} & {\bf ImageNet-200} \\ \hline
base-CNN    & 95.57 $\pm$ 0.08 & 76.56 $\pm$ 0.16 & 92.97 $\pm$ 0.16 & 87.27 $\pm$ 0.15 & 84.39 $\pm$ 0.07 \\ \hline
Alg. only   & 93.85 $\pm$ 0.63 & 76.64 $\pm$ 0.69 & 92.60 $\pm$ 0.07 & 87.64 $\pm$ 0.30 & 80.58 $\pm$ 0.08 \\ \hline
Arch. only  & 95.68 $\pm$ 0.05 & 76.98 $\pm$ 0.13 & 93.49 $\pm$ 0.03 & 87.40 $\pm$ 0.74 & 84.47 $\pm$ 0.13 \\ \hline
PT2-CNN     & {\bf 95.99} $\pm$ 0.07 & {\bf 79.33} $\pm$ 0.15 & {\bf 93.84} $\pm$ 0.14 & {\bf 88.09} $\pm$ 0.30 & {\bf 85.17} $\pm$ 0.10 \\ \hline
\end{tabular}
\vspace{-0.5cm}
\end{center}
\caption{Accuracy of the ablation study models with the same space and time complexity, these 4 models enable us to evaluate the effect of the \emph{TransNet} architecture as well as the \emph{TransNet} algorithm separately. Mean and standard error for 3 repetitions are shown.}
\label{table:acc_ablation}
\end{table*}

Our method consists of 2 main components - the \emph{TransNet} architecture as well as the training algorithm Alg.~\ref{alg:1}. To evaluate the accuracy gain of each component we consider two variations:
\begin{itemize}[itemsep=1pt]
    \item \textbf{Architecture only}: in this method we train the multi-headed architecture (2 in this case) by feeding each head the same un-transformed batch (equivalent to a \emph{TransNet} model with the multi-set of $\{id,id\}$ transformations). Prediction is retrieved from a single  head (similar to PT2-CNN).
    \item \textbf{Algorithm only}: in this method we train the base (one headed) model by the same algorithm Alg.~\ref{alg:1}. (This model was also considered in the theoretical part~\ref{sec:single_vs_multi_headed}, termed single headed model.)
\end{itemize}

We compare the two methods above to the "base-CNN" regular model and the complete model "PT2-CNN", see Table~\ref{table:acc_ablation}. We can see that using only one of the components doesn't yield any significant accuracy gain. This suggest that the complete model benefits from both components working together: the training algorithm increases the model kernel's invariance on the one hand, while the multi-heads architecture encourage the model to capture meaningful orientation information on the other hand.

\section{Summary}
We introduced a model inspired by self-supervision, which includes a base CNN model attached to multiple heads, each corresponding to a different transformation from a fixed set of transformations. The self-supervised aspect of the model is crucial, as the chosen transformations must not occur naturally in the data.  When the model is pruned back to match the base CNN, it achieves better test accuracy and improved generalization, which  is attributed to the increased invariance of the model's kernels in the last layer. We observed that excess invariance, while improving generalization, eventually curtails the test accuracy. 

We evaluated our model on various image data-sets, observing that each data-set achieves its own optimal kernel's invariance level, \ie there's no optimal number of heads for all data-sets.  Finally, we introduced an invariance score gauge ($IS$), which measures the level of invariance achieved by the model's kernels. $IS$ may be leveraged to determine the optimal invariance level, as well as potentially function as an independent regularization term.

\section*{Acknowledgements}
This work was supported in part by a grant from the Israel Science Foundation (ISF) and by the Gatsby Charitable Foundations.

{\small
\bibliographystyle{ieee_fullname}
\bibliography{main}
}


\appendix
\section*{Appendix}

\section{Implementation details}
\label{app:imp_details}
We employed the ResNet \cite{he2016deep} architecture, specifically the ResNet18 architecture for all the data-sets except for the ImageNet-200 which was evaluated using the ResNet50 architecture.  It's important to notice that we haven't changed the hyper-parameters used by the regular CNN architecture which \emph{TransNet} is based on. This may strengthen the results as one may fine tune these hyper-parameters to suit best the \emph{TransNet} model.

We used a weight decay of 0.0001 and momentum of 0.9. The model was trained with a batch size of 64 for all the data-sets except for ImageNet-200 where we increased the batch size to 128. We trained the model for 300 epochs, starting with a learning rate of 0.1, divided by 10 at the 150 and 225 epochs, except for the ImageNet-200 model which was trained for 120 epochs, starting with a learning rate of 0.1, divided by 10 at the 40 and 80 epochs. We normalized the images as usual by subtracting the image's mean and dividing by the image's standard deviation (color-wise). 

We employed a mild data augmentation scheme - horizontal flip with probability of 0.5. For the CIFAR data-sets we padded each dimension by 4 pixels and cropped randomly (uniform) a 32$\times$32 patch from the enlarged image \cite{lee2015deeply} while for the ImageNet family data-sets we cropped randomly (uniform) a 224$\times$224 patch from the original image.

In test time, we took the original image for the CIFAR data-sets and a center crop for the ImageNet family data-sets. The prediction of each model is the mean of the model's output on the original image and a horizontally flipped version of it. Note that a horizontal flip occurs naturally in every data-set we use for evaluation and therefore isn't associated with any of the \emph{TransNet} model's heads that we evaluate.

\end{document}